\documentclass[11pt,a4paper]{article}

\usepackage[utf8]{inputenc} 
\usepackage[T1]{fontenc}    
\usepackage{hyperref}       
\usepackage{url}            
\usepackage{booktabs}       
\usepackage{amsfonts}       
\usepackage{nicefrac}       
\usepackage{microtype}      
\usepackage{xcolor}         
\usepackage{tikz}
\usepackage{caption}
\usepackage{float}
\usetikzlibrary{arrows.meta}
\usetikzlibrary{calc}
 \usepackage[utf8]{inputenc}   
\usepackage[T1]{fontenc}      

\usepackage{comment}
\usepackage{geometry}
\usepackage{natbib}

\usepackage[tbtags]{amsmath}
\usepackage{amsthm}
\allowdisplaybreaks
\usepackage{amssymb,mathrsfs}
\usepackage{amsfonts}
\usepackage{upgreek}
\usepackage{xspace}

\usepackage{graphicx}
\usepackage{subfig}
\usepackage{color}
\usepackage{algorithm, algorithmic}

\usepackage{stmaryrd}
\usepackage[inline]{enumitem}
\usepackage{url}

\usepackage{tikz}
\usetikzlibrary{calc}

\usepackage{pgfplots}
\usepackage{xcolor}
\usepackage{bbm}
\usepackage{ifthen}
\usepackage{xargs}
\usepackage[textwidth=1.8cm]{todonotes}

\usepackage{aliascnt}
\usepackage{cleveref}
\usepackage{autonum}
\makeatletter
\newtheorem{theorem}{Theorem}
\crefname{theorem}{theorem}{Theorems}
\Crefname{Theorem}{Theorem}{Theorems}

\newtheorem*{lemma_nonumber*}{Lemma}

\newaliascnt{lemma}{theorem}
\newtheorem{lemma}[lemma]{Lemma}
\aliascntresetthe{lemma}
\crefname{lemma}{lemma}{lemmas}
\Crefname{Lemma}{Lemma}{Lemmas}

\newaliascnt{corollary}{theorem}
\newtheorem{corollary}[corollary]{Corollary}
\aliascntresetthe{corollary}
\crefname{corollary}{corollary}{corollaries}
\Crefname{Corollary}{Corollary}{Corollaries}

\newaliascnt{proposition}{theorem}
\newtheorem{proposition}[proposition]{Proposition}
\aliascntresetthe{proposition}
\crefname{proposition}{proposition}{propositions}
\Crefname{Proposition}{Proposition}{Propositions}

\newaliascnt{definition}{theorem}

\aliascntresetthe{definition}
\crefname{definition}{definition}{definitions}
\Crefname{Definition}{Definition}{Definitions}

\newaliascnt{remark}{theorem}
\newtheorem{remark}[remark]{Remark}
\aliascntresetthe{remark}
\crefname{remark}{remark}{remarks}
\Crefname{Remark}{Remark}{Remarks}

\crefname{example}{example}{examples}
\Crefname{Example}{Example}{Examples}

\crefname{technique}{technique}{techniques}
\Crefname{Technique}{Technique}{Techniques}

\crefname{figure}{figure}{figures}
\Crefname{Figure}{Figure}{Figures}

\crefformat{assumption}{{\textbf{A}}#2#1#3}

\newtheorem{assumptionF}{\textbf{F}\hspace{-3pt}}
\crefformat{assumptionF}{{\textbf{F}}#2#1#3}

\Crefname{assumptionB}{\textbf{B}\hspace{-3pt}}{\textbf{B}\hspace{-3pt}}
\crefname{assumptionB}{\textbf{B}}{\textbf{B}}

\Crefname{assumptionC}{\textbf{C}\hspace{-3pt}}{\textbf{C}\hspace{-3pt}}
\crefname{assumptionC}{\textbf{C}}{\textbf{C}}

\Crefname{assumptionH}{\textbf{H}\hspace{-3pt}}{\textbf{H}\hspace{-3pt}}
\crefname{assumptionH}{\textbf{H}}{\textbf{H}}

\Crefname{assumptionT}{\textbf{T}\hspace{-3pt}}{\textbf{T}\hspace{-3pt}}
\crefname{assumptionT}{\textbf{T}}{\textbf{T}}

\Crefname{assumptionT}{\textbf{T}\hspace{-3pt}}{\textbf{T}\hspace{-3pt}}
\crefname{assumptionT}{\textbf{T}}{\textbf{T}}

\Crefname{assumptionL}{\textbf{L}\hspace{-3pt}}{\textbf{L}\hspace{-3pt}}
\crefname{assumptionL}{\textbf{L}}{\textbf{L}}

\Crefname{assumptionQ}{\textbf{Q}\hspace{-3pt}}{\textbf{Q}\hspace{-3pt}}
\crefname{assumptionQ}{\textbf{Q}}{\textbf{Q}}

\Crefname{assumptionAR}{\textbf{AR}\hspace{-3pt}}{\textbf{AR}\hspace{-3pt}}
\crefname{assumptionAR}{\textbf{AR}}{\textbf{AR}}

\usepackage{bm}
\usepackage{wrapfig}

\def\Pens{\mathscr{P}}

\newcommand{\schro}{Schr\"{o}dinger\xspace}

\newcommand{\tta}{\mathtt{A}}

\newcommand{\Capprox}{\tta}

\newcommandx\ctun[1][1=T]{\Capprox_{#1,1}}

\newcommandx{\expec}[2]{{\mathbb E}\left[#1 \middle \vert #2  \right]} 

\newcommand{\transference}{\mathbf{T}}

\newcommand{\rme}{\mathrm{e}}

\newcommand{\Lip}{\mathtt{L}}
\newcommand{\Lipset}{\mathrm{Lip}}

\newcommandx{\norm}[2][1=]{\ifthenelse{\equal{#1}{}}{\left\Vert #2 \right\Vert}{\left\Vert #2 \right\Vert^{#1}}}
\newcommandx{\normLigne}[2][1=]{\ifthenelse{\equal{#1}{}}{\Vert #2 \Vert}{\Vert #2\Vert^{#1}}}

\def\bfc{\mathbf{c}}

\def\msk{\mathsf{K}}

\def\msc{\mathsf{C}}
\def\tmsc{\tilde{\msc}}
\def\mse{\mathsf{E}}
\def\msf{\mathsf{F}}
\def\tmsf{\tilde{\msf}}

\def\msv{\mathsf{V}}

\def\msx{\mathsf{X}}
\def\msz{\mathsf{Z}}
\def\msy{\mathsf{Y}}
\def\ddx{d_\msx}
\def\ddy{d_\msy}

\def\mcy{\mathcal{Y}}
\def\mcx{\mathcal{X}}
\def\mce{\mathcal{E}}

\def\Qbb{\mathbb{Q}}

\def\Pbb{\mathbb{P}}

\def\rset{\mathbb{R}}

\def\nset{\mathbb{N}}

\def\rmd{\mathrm{d}}

\def\rme{\mathrm{e}}

\def\rmc{\mathrm{C}}

\newcommandx{\functionspace}[2][1=+]{\mathbb{F}_{#1}(#2)}

\newcommand{\argmin}{\operatorname*{arg\,min}}

\newcommandx{\VarDeux}[3][3=]{\operatorname{Var}^{#3}_{#1}\left\{#2 \right\}}

\newcommand{\LeftEqNo}{\let\veqno\@@leqno}

\newcommand{\N}{\ensuremath{\mathbb{N}}}

\newcommand{\abs}[1]{\left\vert #1 \right\vert}
\newcommand{\absLigne}[1]{\vert #1 \vert}

\newcommandx{\Vnorm}[2][1=V]{\| #2 \|_{#1}}
\newcommandx{\VnormEq}[2][1=V]{\left\| #2 \right\|_{#1}}

\newcommand{\parenthese}[1]{\left(#1 \right)}

\newcommand{\defEns}[1]{\left\lbrace #1 \right\rbrace }

\newcommandx\probaMarkovTilde[2][2=]
{\ifthenelse{\equal{#2}{}}{{\widetilde{\mathbb{P}}_{#1}}}{\widetilde{\mathbb{P}}_{#1}\left[ #2\right]}}

\def\ie{\textit{i.e.}}

\def\eqsp{\;}
\newcommand{\coint}[1]{\left[#1\right)}

\newcommand{\ooint}[1]{\left(#1\right)}
\newcommand{\ccint}[1]{\left[#1\right]}

\newcommandx{\weight}[2][2=n]{\omega_{#1,#2}^N}

\newcommand{\cball}[2]{\bar{\operatorname{B}}(#1,#2)}

\newcommandx\sequence[3][2=,3=]
{\ifthenelse{\equal{#3}{}}{\ensuremath{\{ #1_{#2}\}}}{\ensuremath{\{ #1_{#2}, \eqsp #2 \in #3 \}}}}

\newcommandx\sequenceD[3][2=,3=]
{\ifthenelse{\equal{#3}{}}{\ensuremath{\{ #1_{#2}\}}}{\ensuremath{( #1)_{ #2 \in #3} }}}

\newcommandx{\sequencen}[2][2=n\in\N]{\ensuremath{\{ #1_n, \eqsp #2 \}}}
\newcommandx\sequenceDouble[4][3=,4=]
{\ifthenelse{\equal{#3}{}}{\ensuremath{\{ (#1_{#3},#2_{#3}) \}}}{\ensuremath{\{  (#1_{#3},#2_{#3}), \eqsp #3 \in #4 \}}}}
\newcommandx{\sequencenDouble}[3][3=n\in\N]{\ensuremath{\{ (#1_{n},#2_{n}), \eqsp #3 \}}}

\newcommand{\opnorm}[1]{{\left\vert\kern-0.25ex\left\vert\kern-0.25ex\left\vert #1
    \right\vert\kern-0.25ex\right\vert\kern-0.25ex\right\vert}}

\def\Lip{\operatorname{Lip}}

\newcommandx{\CPE}[3][1=]{{\mathbb E}_{#1}\left[#2 \middle \vert #3  \right]} 
\newcommandx{\CPELigne}[3][1=]{{\mathbb E}_{#1}[#2  \vert #3  ]} 
\newcommandx{\CPEsq}[3][1=]{{\mathbb{E}^{1/2}}_{#1}\left[#2 \middle \vert #3  \right]} 
\newcommandx{\CPVar}[3][1=]{\mathrm{Var}^{#3}_{#1}\left\{ #2 \right\}}
\newcommand{\CPP}[3][]
{\ifthenelse{\equal{#1}{}}{{\mathbb P}\left(\left. #2 \, \right| #3 \right)}{{\mathbb P}_{#1}\left(\left. #2 \, \right | #3 \right)}}

\newcommandx{\osc}[2][1=]{\mathrm{osc}_{#1}(#2)}

\newcommand{\ensemble}[2]{\left\{#1\,:\eqsp #2\right\}}
\newcommand{\ensembleLigne}[2]{\{#1\,:\eqsp #2\}}

\def\rmD{\mathrm{D}}

\newcommand\coupling[2]{\Gamma(\mu,\nu)}

\def\diam{\mathfrak{d}}

\def\vareps{\varepsilon}

\newcommandx{\KL}[2]{\operatorname{KL}\left( #1 | #2 \right)}
\newcommandx{\KLsqrt}[2]{\operatorname{KL}^{1/2}\left( #1 | #2 \right)}
\newcommandx{\Jef}[2]{\operatorname{J}\left( #1 , #2 \right)}
\newcommandx{\JefLigne}[2]{\operatorname{J}( #1 , #2 )}
\newcommandx{\KLLigne}[2]{\operatorname{KL}( #1 | #2 )}

\def\gaStep
\def\QKer{Q}

\def\distance{\mathbf{d}}
\newcommandx{\wasserstein}[3][1=\distance,3=]{\mathbf{W}_{#1}^{#3}\left(#2\right)}
\newcommandx{\wassersteinLigne}[3][1=\distance,3=]{\mathbf{W}_{#1}^{#3}(#2)}
\newcommandx{\wassersteinD}[1][1=\distance]{\mathbf{W}_{#1}}
\newcommandx{\wassersteinDLigne}[1][1=\distance]{\mathbf{W}_{#1}}

\def\sigmaD{\sigma^2}

\newcommandx{\phibfs}[1][1=]{\pmb{\varphi}_{\sigmaD_{#1}}}

\newcommandx\sequenceg[3][2=,3=]
{\ifthenelse{\equal{#3}{}}{\ensuremath{( #1_{#2})}}{\ensuremath{( #1_{#2})_{ #2 \geq #3}}}}

\newcommandx{\distV}[1][1=\bfc]{\mathbf{W}_{#1}}
\newcommandx{\distVdeux}[1][1=W_2]{\mathbf{d}_{#1}}

 \usepackage{comment}
 \usepackage{authblk}
 \usepackage{cancel}
\makeatletter
\renewcommand\AB@affilsepx{, \protect\Affilfont}
\makeatother
\providecommand{\keywords}[1]
{
  \small	
  \textbf{\textit{Keywords---}} #1
}
\hypersetup{colorlinks,citecolor=blue}

\title{Quantitative Uniform Stability of the Iterative Proportional Fitting Procedure}

\author{George Deligiannidis\thanks{corresponding author: deligian@stats.ox.ac.uk}~, Valentin De Bortoli\thanks{valentin.debortoli@gmail.com}~, Arnaud Doucet\thanks{doucet@stats.ox.ac.uk}}
\affil{Department of Statistics, University of Oxford, UK}

\begin{document}

\maketitle

\begin{abstract}
 We establish that the iterates of the Iterative Proportional Fitting
  Procedure, also known as Sinkhorn's algorithm and commonly used to solve
  entropy-regularised Optimal Transport problems, are stable w.r.t.\
  perturbations of the marginals, uniformly in time. Our result is quantitative
  and stated in terms of the 1-Wasserstein metric. As a corollary we establish a
  quantitative stability result for Schr\"odinger bridges.
\end{abstract}

\section{Introduction}
\label{sec:introduction}

The basic problem of Optimal Transport (OT) (see \cite{villani2009optimal} for a broad overview),
in its modern formulation introduced by \cite{kantorovich1942transfer}, 
is to find a \emph{coupling} of two distributions $\mu, \nu$ that minimises
\begin{equation}
    \label{eq:OTproblem}
    \tag*{\textsf{OT}$(\mu,\nu)$}
    \inf_{\pi \in \transference(\mu, \nu)} \int \|x-y\|^2 \rmd \pi(x, y)   , 
\end{equation}
where $\transference(\mu,\nu)$ denotes the collection of probability measures with
marginals $\mu, \nu$ and the Euclidean distance $\|x-y\|$ may be replaced by any
other metric or cost function $c(x,y)$.  OT provides a theoretical framework for
analysis in the space of probability measures and has deep connections with many
branches of mathematics including partial differential equations and
probability. Beyond its intrinsic interest, OT has recently become an extremely
important tool for data science and machine learning, finding numerous
applications in fields as diverse as imaging, computer vision or natural
language processing \citep{peyre2019computational}.

This ubiquity of OT in modern applications is largely due to the computational tractability of the \emph{Entropy-Regularised} Optimal Transport problem
\begin{equation}
    \label{eq:RegOTproblem}\tag*{\textsf{OT}$_\vareps(\mu,\nu)$}
    \inf_{\pi \in \transference(\mu, \nu)} \int  \|x-y\|^2 \rmd \pi( x,  y) +\vareps \KL{\pi}{\mu\otimes \nu}  , 
\end{equation}
which is equivalent to the \emph{static Schr\"odinger bridge}, a problem
going back to \cite{schrodinger1931umkehrung}, see \Cref{eq:schrodinger_bridge}
in Section \ref{sec:main-results}.  Here $\KL{\pi}{\rho}$ denotes the
\emph{Kullback--Leibler divergence} between the probability measures $\pi$ and
$\rho$, defined as
\begin{equation}
\KL{\pi}{\rho} = \begin{cases}
                        \int  \log (\frac{\rmd \pi}{\rmd \rho}(x)) \rmd\pi(x)  , & \pi \ll \rho  ,\\
                        +\infty  ,  & \text{otherwise}  . 
                        \end{cases}
 \end{equation}
 The great interest in \textsf{OT}$_\vareps(\mu,\nu)$, as explained in the seminal paper of \cite{cuturi2013sinkhorn},
 stems from its amenability to the \emph{Iterative Proportional
   Fitting Procedure} (IPFP), see \Cref{eq:ipfp} below. The theoretical properties of IPFP, also known as the Sinkhorn algorithm,
 have been investigated in numerous works, and are therefore fairly well
 understood.

 Due to its computational tractability, \textsf{OT}$_\vareps(\mu,\nu)$ has been
 used in applications as an approximation to \textsf{OT}$(\mu,\nu)$. Rigorous
 justification of this approximation has been the subject of intense research
 recently. Indeed it has been established, see e.g.\
 \cite{cominetti1994asymptotic,mikami2004monge,leonard2012schrodinger,carlier2017convergence},
 that as the regularisation parameter $\vareps \to 0$, the solution of
 \textsf{OT}$_\vareps(\mu,\nu)$ converges to that of \textsf{OT}$(\mu,\nu)$.

More recently however, Schr\"odinger bridges and entropy-regularised OT are
being studied for their own sake, finding applications in control, computational
statistics and machine learning, see e.g.\
\cite{bernton2019schr,chen2021optimal,corenflos2021differentiable,de2021diffusion,huang2021schrodinger,li2020continuous,vargas2021solving}. In these applications, the entropy
regularisation may be a desirable feature rather than an approximation, and the
main source of error is the fact that the marginal distributions are typically
intractable and often approximated by empirical versions. It is then desirable
that as the number of samples increases, this error vanishes. For example, a
quantitative version of this statement, can then be used to establish that the
differentiable particle filter proposed in \cite{corenflos2021differentiable}
converges as the sample size increases, for any $\vareps>0$, thus strengthening
the analysis of \cite{corenflos2021differentiable} which requires
$\vareps_N\to 0$ as $N\to \infty$ to ensure consistency towards the true optimal
filter.

This is the question we study in this paper. In particular we establish the
stability of the IPFP and of the solution of the corresponding Schr\"odinger
bridge problem w.r.t. perturbations of the marginals.

For standard OT, a classical argument 
using compactness  and cyclical monotonicity guarantees a qualitative version of this result, see e.g.\cite[Theorem~5.23, Corollary~5.23]{villani2009optimal}. Quantitative versions of this result appeared much more recently, at least in the case of quadratic costs, in \cite{merigot2020quantitative}, \cite{li2020quantitative}, \cite{delalande2021quantitative}. In particular it is established that the optimal transport plans, or maps in the case of absolutely continuous measures, is H\"older continuous in the marginals, with exponent $1/2$ w.r.t.\ the marginals. 
It is also known that the exponent $1/2$ is the best possible, see \cite{gigli2011holder}.

For entropy-regularised OT and the static Schr\"odinger bridge problem, the
first qualitative result appeared very recently in \cite{ghosal2021stability},
based on a version of cyclical monotonicity for entropy-regularised OT
introduced by \cite{bernton2021entropic}.  In the quantitative direction,
\cite{luise2019sinkhorn} establish Lipschitz continuity of the potentials
w.r.t.\ the marginals, measured in the total variation metric, which is too
strong to capture the situation where the marginals are being approximated by
empirical versions. For \emph{smooth} cost functions, \cite{luise2019sinkhorn}
also establish that the sample complexity of learning the potentials is of order
$n^2$, leveraging results from \cite{genevay2019sample} on the regularity of
potentials and the duality between Maximum Mean Discrepancy type metrics and
Sobolev spaces. However, if one is interested in learning the Schr\"odinger
bridge the situation is more complicated; the Wasserstein-1 distance between two
couplings is lower bounded by the distance of the marginals and so the results
by \cite{fournier2015rate} imply that the sample complexity of learning the
Schr\"odinger bridge must scale at least as $n^{d}$ on $\mathbb{R}^d$.

We present here the first, to the best of our knowledge, quantitative stability
result for entropy-regularised OT. In particular, this follows from a stronger
result, namely the uniform in time stability of IPFP, that is the Sinkhorn
iterates, w.r.t. the marginal distributions. We think this result is of particular importance for practical applications as IFPF is typically used for a finite number of iterations to approximate the Schr\"odinger bridge. One interesting fact is that in
contrast to the standard OT problem, the solution of the entropy-regularised
problem is Lipschitz continuous, in the Wasserstein metric, w.r.t.\ the
marginals. However, as the regularisation parameter $\vareps$ vanishes, the
Lipschitz constant blows up as expected by the H\"older continuity of the OT
plan.



The recent paper by \cite{eckstein2021quantitative}, which appeared a
 couple of months after the first version of the present manuscript, uses very
 interesting methods, completely different to the ones in the present paper, to
  establish the quantitative stability of the Schr\"odinger bridge w.r.t.\ the
  marginals measured in the Wasserstein distance. On the one hand, the setting
  of \cite{eckstein2021quantitative} is more general than ours, as it does not
  require compactness. On the other hand, \cite{eckstein2021quantitative} only
  establish the stability of the Schr\"odinger bridge instead of the full iterates of IPFP and H\"older continuity w.r.t.\ the marginals with exponent $1/2$.

\section{Notation}
\label{sec:notation}
For a metric space $(\msz, d_\msz)$, we write $\mathscr{B}(\msz)$ for the Borel
$\sigma$-algebra on $\msz$ and $\diam_\msz$ for the diameter of $\msz$, that is
$\diam_\msz= \sup\{d_\msz(z,z'): z,z' \in \msz\}$. We also write 
$\Pens(\msz)$ to denote the
subspace of Borel probability measures.  For $\pi \in \Pens(\msx)$, we define
the support of $\pi$ as
$$\mathsf{supp}(\pi)= \left\{A\in \mathscr{B}(\msx): \text{$A$ is closed,
    $\pi(A^\textsf{c})=0$} \right\}.$$ For two metric spaces $(\msx, d_\msx)$,
$(\msy, d_\msy)$, $\Pens(\msx\times\msy)$ is always defined w.r.t.\ the product $\sigma$-algebra. For $\Pbb\in \Pens(\msx\times\msy)$, we will write $\Pbb_0$, $\Pbb_1$ to denote the
first and second marginals respectively. For
$\mu\in \Pens(\msx), \nu \in \Pens(\msy)$, we let
$$\transference(\mu, \nu) = \{\mathbb{P}\in \Pens(\msx \times \msy): \mathbb{P}_0=\mu,\, \mathbb{P}_1=\nu\}.$$
For a function $f:\msx \to \mathbb{R}^d$, we write
$\|f\|_\infty= \sup_{x\in \msx} \| f(x)\|$, where $\|\cdot\|$ denotes the usual
Euclidean norm. For a function $f:\msx \to \msy$, we define its Lipschitz constant $\Lip(f)$ by
\begin{equation}
  \Lip(f) = \inf \ensembleLigne{C \geq 0}{d_{\msy}(f(x_0), f(x_1)) \leq C d_{\msx}(x_0,x_1)  , \eqsp x_0, x_1 \in \msx} . 
\end{equation}
We also define 
\begin{equation}
    \Lip(\msx, \msy) = \{f: \msx \to \msy: \Lip(f)<\infty\}, \qquad
    \Lip_1(\msx, \msy) = \{f: \msx \to \msy: \Lip(f)\leq 1\}, 
\end{equation}
and write $\rmc(\msx, \msy)$ for the class of continuous functions from $\msx$
to $\msy$.

\section{Main results}
\label{sec:main-results}
Let $(\msx, \ddx)$, $(\msy, \ddy)$ be two compact metric spaces and
write $\mcx, \mcy$ for their respective Borel $\sigma$-algebras. We will use $d$
to denote the metric for both $\msx, \msy$ when the context allows.  Let
$\pi_0 \in \Pens(\msx), \pi_1 \in \Pens(\msy)$.  We begin by recalling the
Iterative Proportional Fitting Procedure (IPFP) solving the following \schro
bridge problem
\begin{equation}
  \label{eq:schrodinger_bridge}
\Pbb^\star \in \argmin \ensembleLigne{\KLLigne{\Pbb}{\Qbb}}{\Pbb \in \Pens(\msx \times \msy), \eqsp  \Pbb_0 = \pi_0 \eqsp , \Pbb_1 = \pi_1},
\end{equation}
where $\Qbb \in \Pens(\msx \times \msy)$ is a reference measure
admitting a density w.r.t.\ $\rho_0 \otimes \rho_1$, with $\rho_0 \in \Pens(\msx)$ equivalent to $\pi_0$, and $\rho_1\in\Pens(\msy)$, equivalent to $\pi_1$; that is for any
$(x,y) \in \msx \times \msy$
\begin{equation}
  \label{eq:ref_measure}
  \rmd \Qbb / \rmd (\rho_0 \otimes \rho_1)(x,y) = K(x,y) = \exp[-c(x,y)].
\end{equation}
In the case where
$\msx=\msy$, we have that Problem~\eqref{eq:schrodinger_bridge} with the choice
$c(x,y)= \|x-y\|^2/\vareps$ is equivalent to \ref{eq:RegOTproblem}, see e.g.
\cite[Remark 4.2]{peyre2019computational}. 
First, we give a sufficient condition to ensure that the solution of \eqref{eq:schrodinger_bridge} exists and is unique.
The proof of this proposition is a straightforward consequence of \cite[Corollary 3.2]{csiszar1975divergence}.
\begin{proposition}
  Assume that $\KL{\pi_i}{\rho_i} < +\infty$ for $i \in \{0,1\}$ and that
  $c \in \rmc(\msx \times \msy, \rset)$. Then there exists a unique solution to
  \eqref{eq:schrodinger_bridge}.
\end{proposition}
The following proposition, see
\cite[Proposition 4.2]{peyre2019computational} for instance, ensures that we can
assume without loss of generality that $\rho_0 = \pi_0$ and $\rho_1 = \pi_1$.
 
\begin{proposition}
  Assume that $\KL{\pi_i}{\rho_i}<+\infty$ for $i \in \{0,1\}$ and that
  $c \in \rmc(\msx \times \msy, \rset)$.  Let $\Pbb^\star$ solution of
  \eqref{eq:schrodinger_bridge} with $\Qbb$ given by \eqref{eq:ref_measure} and
  $\hat{\Pbb}^\star$ the solution of \eqref{eq:schrodinger_bridge} with $\Qbb$
  such that for any $(x,y) \in \msx \times \msy$
\begin{equation}
  \label{eq:ref_measure_duo}
  \rmd \Qbb / \rmd (\pi_0 \otimes \pi_1)(x,y) = K(x,y).
\end{equation}
Then $\Pbb^\star = \hat{\Pbb}^\star$.
\end{proposition}
As a consequence, for the rest of this paper, we assume that $\rho_0 = \pi_0$
and $\rho_1 = \pi_1$.  In order to solve \eqref{eq:schrodinger_bridge} we
consider the IPFP sequence which iteratively solves each half-bridge problem,
\ie \ we define $(\Pbb^n)_{n \in \nset}$ such that for any $n \in \nset$
\begin{align}
  \label{eq:ipfp}
  &\Pbb^{2n+1} = \argmin \ensembleLigne{\KL{\Pbb}{\Pbb^{2n}}}{\Pbb \in \Pens(\msx \times \msy), \Pbb_0 = \pi_0}   , \\
  &\Pbb^{2n+2} = \argmin \ensembleLigne{\KL{\Pbb}{\Pbb^{2n+1}}}{\Pbb \in \Pens(\msx \times \msy)  , \Pbb_1 = \pi_1}  , 
\end{align}
with $\Pbb^0 = \Qbb$ and where we recall that $\Pbb_0, \Pbb_1$ denote the
marginals of the joint distribution $\Pbb$. Note that $(\Pbb^n)_{n \in \nset}$
is uniquely defined if $c \in \rmc(\msx \times \msy, \rset)$, see \cite[Theorem
3.1]{csiszar1975divergence}.   For discrete or compact spaces
  it is known that IPFP converges at an exponential rate on compact or discrete
  spaces; see e.g.\
  \cite{chen2016entropic,altschuler2017near,franklin1989scaling}. For the
  non-compact case, convergence, but without any rates, has been established
  under various regularity conditions in \cite{ruschendorf1995convergence}. 

We are now ready to state our main result which is a quantitative
uniform stability estimate for the IPFP.

\begin{theorem}
  \label{thm:stability_ipfp}
  Assume that $c  \in \Lip(\msx\times\msy, \mathbb{R})$.
  For any $\pi_0, \hat{\pi}_0 \in \Pens(\msx)$,
  $\pi_1, \hat{\pi}_1 \in \Pens(\msy)$ let $(\Pbb^{n})_{n\in \nset}$ and
  $(\hat{\Pbb}^{n})_{n\in \nset}$ the IPFP sequence with marginals
  $(\pi_0, \pi_1)$ respectively $(\hat{\pi}_0, \hat{\pi}_1)$. Then for any $n \in \nset$ we have
  \begin{equation}
    \wassersteinD[1](\Pbb^n, \hat{\Pbb}^n) \leq C \defEns{\wassersteinD[1](\pi_0, \hat{\pi}_0) + \wassersteinD[1](\pi_1, \hat{\pi}_1)}  ,
  \end{equation}
  with
  \begin{equation}
    C = \rme^{17 \normLigne{c}_\infty} \{1 + 15\Lip(c) (\diam_\msx + \diam_\msy) \} .
  \end{equation}
\end{theorem}

\Cref{thm:stability_schro}, establishing the quantitative stability of the
\schro bridge, can be obtained by making minor modifications to the proof
of \Cref{thm:stability_ipfp}.

\begin{corollary}
  \label{thm:stability_schro}
  Assume that $c  \in \Lip(\msx\times\msy, \mathbb{R})$.
  For any $\pi_0, \hat{\pi}_0 \in \Pens(\msx)$,
  $\pi_1, \hat{\pi}_1 \in \Pens(\msy)$ let $\Pbb^\star$, respectively
  $\hat{\Pbb}^\star$, be the \schro bridge with marginals $(\pi_0, \pi_1)$,
  respectively $(\hat{\pi}_0, \hat{\pi}_1)$. Then, we have 
  \begin{equation}
    \wassersteinD[1](\Pbb^\star, \hat{\Pbb}^\star) \leq C \defEns{\wassersteinD[1](\pi_0, \hat{\pi}_0) + \wassersteinD[1](\pi_1, \hat{\pi}_1)},
  \end{equation}
  with $C$ as in \Cref{thm:stability_ipfp}.
\end{corollary}

\begin{remark}
  Although the constants are far from sharp, Lipschitz continuity in the marginals is the best one can hope.  Indeed, for any $\mathbb{P}\in \transference(\pi_0, \pi_1)$,
  $\hat{\mathbb{P}}\in \transference(\hat{\pi}_0, \hat{\pi}_1)$ we have that
\begin{align}
\wassersteinD[1](\mathbb{P}, \hat{\mathbb{P}})
  &= \sup \ensemble{\int_{\msx \times \msy} f(x,y)  \rmd \mathbb{P}(x,y) - \int_{\msx \times \msy} f(x,y) \rmd \hat{\mathbb{P}}(x,y)}{f \in \Lip_1(\msx\times \msy)}\\
  &\geq \sup \ensemble{\int_{\msx \times \msy} f(x)  \rmd \mathbb{P}(x,y) - \int_{\msx \times \msy} f(x) \rmd \hat{\mathbb{P}}(x,y)}{f \in \Lip_1(\msx)} \geq  \wassersteinD[1](\pi_0, \hat{\pi}_0),
\end{align}
and a similar calculation shows that $\wassersteinD[1](\mathbb{P}, \hat{\mathbb{P}})\geq \min\{\wassersteinD[1](\pi_0, \hat{\pi}_0), \wassersteinD[1](\pi_1, \hat{\pi}_1)\}$. In the case where $\hat{\pi}_0, \hat{\pi}_1$, are  empirical versions of $\pi_0, \pi_1$ respectively with $n$ samples, the Lipschitz continuity in the marginals and the results by \cite{fournier2015rate} also imply a sample complexity of $n^{d}$ for learning the Schr\"odinger bridge when $\msx=\msy=\mathbb{R}^d$.
\end{remark}
\section{Proof}
\label{sec:proof}

The proof is divided into four parts. First, we recall that the IPFP sequence is
associated with a sequence of potentials. In \Cref{sec:prop-extens-potent} we
show quantitative regularity and boundedness properties for these
potentials. The boundedness is due to a reparameterization by
\cite{carlier2020differential}. Then, in \Cref{sec:hilb-birkh-metr} we recall a
contraction property and show useful Lipschitz properties of the potentials
w.r.t. the Hilbert--Birkhoff metric. We then turn to the proof of the uniform
quantitative stability of the potentials w.r.t.\ this metric in
\Cref{sec:quant-unif-bounds}. Finally, in \Cref{sec:from-potent-prob} we show
how uniform quantitative bounds on the potentials translate into bounds onto
probability measures which concludes the proof.

\subsection{Regularity properties of the potentials}
\label{sec:prop-extens-potent}

In this section, we fix $\pi_0 \in \Pens(\msx)$ and $\pi_1 \in \Pens(\msy)$ and
let $(\Pbb^n)_{n \in \nset}$ the IPFP sequence given by \eqref{eq:ipfp}. The
IPFP sequence can be described by a corresponding sequence of (measurable)
potentials $(\tilde f_n, \tilde g_n)_{n \in \nset}$ such that for any
$n \in \nset$, $\tilde{f}_n: \ \msx \to \ooint{0,+\infty}$,
$\tilde{g}_n: \ \msy \to \ooint{0,+\infty}$ and $\tilde{f}_0 = \tilde{g}_0 = 1$,
see \cite[Theorem 3.1]{csiszar1975divergence}.

\begin{proposition}
  For any $n \in \nset$ and $(x, y) \in \msx \times \msy$ we have
  \begin{align}
    &(\rmd \Pbb^{2n} / \rmd \pi_0 \otimes \pi_1)(x,y) = \tilde{f}_n(x) K(x,y) \tilde{g}_n(y), \\
    &(\rmd \Pbb^{2n+1} / \rmd \pi_0 \otimes \pi_1)(x,y) = \tilde{f}_{n+1}(x) K(x,y) \tilde{g}_n(y), \\
    &\textstyle{\tilde{f}_{n+1}(x) = \parenthese{\int_{\msy} K(x,y) \tilde{g}_n(y) \rmd \pi_1(y)}^{-1},} \\
    &\textstyle{\tilde{g}_{n+1}(y) = \parenthese{\int_{\msx} K(x,y) \tilde{f}_{n+1}(x) \rmd \pi_0(x)}^{-1}.} 
  \end{align}
\end{proposition}
For any $n \in \nset$, $a_n > 0$ and $(x,y) \in \msx \times \msy$ we have also
\begin{equation}
  (\rmd \Pbb^{2n} / \rmd \pi_0 \otimes \pi_1)(x,y) = (a_n \tilde{f}_n(x)) K(x,y) (\tilde{g}_n(y)/a_n).
\end{equation}
In other words, the measure $\Pbb^{2n}$ is invariant w.r.t.\ rescaling of the
potentials $\tilde{f}_n$ and $\tilde{g}_n$. This observation is at the core of
the work of \cite{carlier2020differential} which proves the geometric
convergence of the IPFP w.r.t.\ the $\mathrm{L}^p$ metric for bounded costs.
For any $n \in \nset$, let $\tilde{\varphi}_n = \log(\tilde{f}_n)$ and
$\tilde{\Psi}_n = \log(\tilde{g}_n)$ and let
$a_n = \exp[-\int_{\msx} \tilde{\varphi}_n(x) \rmd \pi_0(x)]$. Finally, for any
$n \in \nset$, let $\varphi_n = \tilde{\varphi}_n + \log(a_n)$ and
$\Psi_n = \tilde{\Psi}_n - \log(a_n)$. Similarly, for any $n \in \nset$ we define
\begin{equation}
f_n = \exp[\varphi_n]  \eqsp , \qquad g_n = \exp[\Psi_n] \eqsp . 
\end{equation}
The log-potentials $(\varphi_n,\Psi_n)_{n \in \nset}$ can be computed
recursively using the following proposition.

\begin{proposition}
  \label{prop:potential_rescale}
  For any $n \in \nset$ and $(x,y) \in \msx \times \msy$ we have
  \begin{align}
    &\textstyle{\varphi_{n+1}(x) = -\log \left\{\int_{\msy} K(x,y) \exp[\Psi_n(y)] \rmd \pi_1(y)\right\} } \\
     &\qquad \qquad \qquad   \textstyle{+ \int_{\msx} \log \{ \int_{\msy} K(x,y) \exp[\Psi_n(y)] \rmd \pi_1(y) \} \rmd \pi_0(x), }  \\
    &\textstyle{\Psi_{n+1}(y) = -\log \left\{\int_{\msx} K(x,y) \exp[\varphi_{n+1}(y)] \rmd \pi_0(x)\right\},} \\
    &(\rmd \Pbb^{2n}/ \rmd (\pi_0\otimes\pi_1))(x,y) = \exp[\varphi_n(x) + \Psi_n(y)] K(x,y). 
  \end{align}
\end{proposition}
Recall that for any $x, y \in \msx \times \msy$ we have $K(x,y) = \exp[-c(x,y)]$.
Using \cite[Lemma 3.1]{carlier2020differential} we have the following result.

\begin{proposition}
  \label{prop:bound_0}
  For any $n \in \nset$ we have
  $\max(\normLigne{\varphi_n}_{\infty}, \normLigne{\Psi_n}_{\infty}) \leq 3
  \normLigne{c}_{\infty}$. 
\end{proposition}

We now establish the Lipschitz property of these potentials under the assumption that the cost function $c$ is Lipschitz; this is automatically satisfied in the case where $c(x,y)=\|x-y\|^2/\vareps$ and $\msx,\msy$ are compact, or when $c$ is a metric by the triangle inequality.



\begin{proposition}
  \label{prop:bound_1}
  Assume that $c \in\Lip(\msx\times\msy, \mathbb{R})$. Then, for any
  $n \in \nset$,
  \begin{equation}
    \max\{ \Lip( \varphi_{n}), \Lip(\Psi_{n}) \} \leq \Lip( c). 
  \end{equation}
\end{proposition}

\begin{proof}
  Using \Cref{prop:potential_rescale} and the fact that
  $c \in \Lip(\msx \times \msy, \rset)$, we have for any $x, x' \in \msx$
  \begin{align}
  \lefteqn{\varphi_{n+1}(x)-\varphi_{n+1}(x')}\\
  &=\log\left\{  \int_{\msy} K(x',y) \exp[\Psi_n(y)] \rmd \pi_1(y)/ \int_{\msy} K(x,y) \exp[\Psi_n(y)] \rmd \pi_1(y) \right\}\\
  &= \log \left\{ 
  \int_{\msy} \exp[-c(x,y) + c(x,y)-c(x',y)+\Psi_n(y)] \rmd \pi_1(y)\right\} \\
  &\qquad -\log\left\{\int_{\msy} \exp[-c(x,y+\Psi_n(y)] \rmd \pi_1(y)\right\} \\
  &\leq \log \left\{ 
  \int_{\msy} \exp[-c(x,y) + \Lip(c)\ddx(x,x')+\Psi_n(y)] \rmd \pi_1(y)
  \right\} \\
  &\qquad -\log \left\{\int_{\msy} \exp[-c(x,y)+\Psi_n(y)] \rmd \pi_1(y)\right\} \\ 
  &\leq \Lip(c)\ddx (x,x').
  \end{align}
  Similarly we obtain that for any $x, x' \in \msx$,
  $\varphi_{n+1}(x')-\varphi_{n+1}(x)\leq \Lip(c)\ddx (x,x')$, whence it follows
  that $\Lip(\varphi_{n+1})\leq \Lip(c)$.
  Similarly we have that for any $y, y' \in \msy$
  \begin{align}
     \lefteqn{\Psi_{n+1}(y')-\Psi_{n+1}(y)}\\
     &= \log \left \{ \int_{\msx} K(x,y) \exp[\varphi_{n+1}(y)] \rmd \pi_0(x) /
     \int_{\msx} K(x,y') \exp[\varphi_{n+1}(y)] \rmd \pi_0(x)
     \right\}\\
     &= \log \left \{ \int_{\msx}  \exp[-c(x,y')+c(x,y')-c(x,y)+\varphi_{n+1}(y)] \rmd \pi_0(x)\right\}\\
     &\qquad -\log\left\{
     \int_{\msx} \exp[-c(x,y')+\varphi_{n+1}(y)] \rmd \pi_0(x)
     \right\}\\
     &\leq \log \left \{ \int_{\msx}  \exp[-c(x,y')+\Lip(c)\ddy(y,y')+\varphi_{n+1}(y)] \rmd \pi_0(x)\right\}\\
     &\qquad - \log \left\{\int_{\msx} \exp[-c(x,y')+\varphi_{n+1}(y)] \rmd \pi_0(x)
     \right\}\\
     &\leq \Lip(c) \ddy(y,y').\qedhere
  \end{align}
\end{proof}
\begin{remark}\label{rem:lipconstant}
  Notice that for any $n \in \nset$, $\Lip(\varphi_{n}), \Lip(\Psi_{n})$ are
  independent of the functions $\varphi_n, \Psi_n$ and only depend on the
  properties of the kernel $K(\cdot, \cdot)$. 
\end{remark}

Upon combining the two previous propositions and the fact that for any $s,t \in \ccint{0,M}$, $\absLigne{\rme^t - \rme^s} \leq \rme^M \absLigne{t - s}$, we obtain the following controls on
the sequence of rescaled potentials $(f_n, g_n)_{n \in \nset}$.

\begin{proposition}
  \label{prop:control_potential}
  Assume that $c \in\Lip(\msx\times\msy, \mathbb{R})$. Then, we have that for any $n \in \nset$
  \begin{equation}
    \max(\normLigne{f_n}_\infty, \normLigne{g_n}_\infty) \leq \rme^{3 \normLigne{c}_\infty} , \qquad \max(\Lip(f_n), \Lip(g_n)) \leq \Lip(c) \rme^{3 \normLigne{c}_\infty}.
  \end{equation}
\end{proposition}

\subsection{The Hilbert--Birkhoff metric, contraction and Lipschitz properties}
\label{sec:hilb-birkh-metr}

We now recall basic properties of the Hilbert--Birkhoff metric; see \cite{lemmens2013birkhoff,kohlberg1982contraction,bushell1973hilbert} for a review. Let $\mse$ be a real vector space and
$\msk$ a cone in this vector space, \ie \  $\msk$ is convex, $\msk \cap (-\msk) = \{0\}$ and $\lambda \msk \subset \msk$ for any
$\lambda \geq 0$. In what
follows, we let $\msc$ be a part of the cone \ie \ for any $x, y \in \msc$ there
exist $\alpha, \beta \geq 0$ such that $\alpha x - y \in \msk$ and
$\beta y -x \in \msk$. In addition, assume that $\msc$ is convex and that
for any $\lambda > 0$ $\lambda \msk \subset \msk$. In this case we have for
any $x, y \in \msc$ that
\begin{equation}
  M(x,y) = \inf \ensembleLigne{\beta \geq 0}{\beta y - x \in \msk} > 0. 
\end{equation}
Similarly we define for any $x, y \in \msc$
\begin{equation}
  m(x,y) = \sup \ensembleLigne{\alpha \geq 0}{x - \alpha y \in \msk}. 
\end{equation}
Note that $m(x,y) = M(y,x)^{-1} > 0$. Finally, the Hilbert--Birkhoff metric is
defined for any $x, y \in \msc$ by
\begin{equation}
  d_H(x,y) = \log(M(x,y) / m(x,y)). 
\end{equation}
By \cite[Lemma 2.1]{lemmens2013birkhoff}, $d_H$ is a metric on $\msc/\sim$ the
space $\msc$ quotiented by the equivalence relation: $x \sim y$ if there exists
$\lambda > 0$ such that $y = \lambda x$. In particular, if $\normLigne{\cdot}$
is a norm on $\msv$ then letting
$\tilde{\msc} = \ensembleLigne{x \in \msc}{\normLigne{x}=1}$, we have that
$(\tmsc,d_H)$ is a metric space.

Let $(\msv, \normLigne{\cdot})$ and $(\msv', \normLigne{\cdot}')$ be two normed
real vector space and $\msk \subset \msv$, $\msk' \subset \msv'$ two cones. In
addition, let $\msc$ and $\msc'$ be convex parts of $\msk$ and 
$\msk'$ respectively, such that for any $\lambda > 0$, $\lambda \msc \subset \msc$ and
$\lambda \msc' \subset \msc'$. Let $u: \ \msv \to \msv'$ be a linear mapping
such that $u(\msc) \subset \msc'$. The projective diameter of $u$ is given by
\begin{equation}
  \Delta(u) = \sup \ensembleLigne{d_H(u(x), u(y))}{x, y \in \tilde{\msc}}. 
\end{equation}
Similarly, we also define the Birkhoff contraction ratio of $u$
\begin{equation}
  \kappa(u) = \sup \ensembleLigne{\kappa}{d_H(u(x), u(y)) \leq \kappa d_H(x,y),  x,y \in \tmsc}. 
\end{equation}
Using the Birkhoff contraction theorem \citep{birkhoff1957extensions,bauer1965elementary,hopf1963inequality} we have that
\begin{equation}
  \label{eq:uno}
  \kappa(u) \leq \tanh(\Delta(u)/4). 
\end{equation}
In order to use the Birkhoff contraction theorem, we collect the following basic
facts on cones in function spaces.
\begin{proposition}
  \label{prop:hilbert_birkhoff}
  Let $\msz$ be a compact space. $\msf = \coint{0,+\infty}^\msz$ is a cone and
  $\tmsf = \rmc(\msz, \ooint{0,+\infty})$ is a convex part of $\msf$ such that
  for any $\lambda > 0$, $\lambda \tmsf \subset \tmsf$. In addition, we have
  that for any $f, g \in \tmsf$
  \begin{equation}
    d_H(f,g) = \log(\normLigne{f/g}_{\infty}) + \log(\normLigne{g/f}_{\infty}). 
  \end{equation}
\end{proposition}
In this case, we have that for any $f, g \in \tmsf$, $d_H(f,g)$ is the
oscillation of $\log(f/g)$.  In what follows, we introduce key mappings which
allow us to compute the IPFP potential $(f_n)_{n \in \nset}$ and
$(g_n)_{n \in \nset}$. Recall that for any $n \in \nset$ we have
\begin{align}
  \label{eq:potentials_rescale_form}
  &\textstyle{f_{n+1}(x) = a_n \parenthese{\int_{\msy} K(x,y) g_n(y) \rmd \pi_1(y)}^{-1},} \\
  &\textstyle{a_n = \exp[\int_{\msx} \log \parenthese{\int_{\msy} K(x,y) g_n(y) \rmd \pi_1(y)} \rmd \pi_0(x)]}, \\
    &\textstyle{g_{n+1}(y) = \parenthese{\int_{\msx} K(x,y) f_{n+1}(x) \rmd \pi_0(x)}^{-1}.} 
\end{align}
Let $\pi_0 \in \Pens(\msx)$ and $\pi_1 \in \Pens(\msy)$. We define
$\mce_{\pi_0}^x$ and $\mce_{\pi_1}^y$ such that for any
$f: \ \msx \to \coint{0,+\infty}$ and $g: \ \msy \to \coint{0,+\infty}$ we have
\begin{equation}
  \textstyle{\mce_{\pi_0}^x(f)(y) = \int_{\msx} K(x,y) f(x) \rmd \pi_0(x)  , \quad \mce_{\pi_1}^y(g)(x) = \int_{\msy} K(x,y) g(y) \rmd \pi_1(y).}
\end{equation}
The following proposition is a consequence of the Birkhoff contraction theorem,
see also \cite{chen2016entropic}.

\begin{proposition}
  \label{prop:birkhoff_contraction}
  For any $\pi_0 \in \Pens(\msx)$ and $\pi_1 \in \Pens(\msy)$, $\mce_{\pi_0}^x(\rmc(\msx, \ooint{0,+\infty})) \subset \Lip(\msy, \ooint{0,+\infty})$ and $\mce_{\pi_1}^y(\rmc(\msy, \ooint{0,+\infty})) \subset \Lip(\msx, \ooint{0,+\infty})$. In addition, we have
  \begin{equation}
    \max(\kappa(\mce_{\pi_0}^x), \kappa(\mce_{\pi_1}^y)) \leq \tanh(\norm{c}_\infty). 
  \end{equation}
\end{proposition}

\begin{proof}
  Let $\pi_0 \in \Pens(\msx)$. Since
  $K:\ \msx \times \msy \to \ooint{0,+\infty}$ is continuous and
  $\msx \times \msy$ is compact we get that for any
  $f \in \rmc(\msx, \ooint{0,+\infty})$,
  $\mce_{\pi_0}^x(f) \in \rmc(\msy, \ooint{0,+\infty})$. In addition, let
  $u \in \rmc(\msy, \ooint{0,+\infty})$ such that for any $y \in \msy$,
  $u(y) = 1$. Then, we have that 
  \begin{align}
    \label{eq:dos}
    &\Delta(\mce_{\pi_0}^x) \leq 2 \sup \ensembleLigne{d_H(\mce_{\pi_0}^x(f), u)}{f \in \rmc(\msx, \ooint{0,+\infty})}, \\
    &\Delta(\mce_{\pi_1}^y) \leq 2 \sup \ensembleLigne{d_H(\mce_{\pi_1}^x(g), u)}{g \in \rmc(\msy, \ooint{0,+\infty})}.
  \end{align}
  In addition, using \Cref{prop:hilbert_birkhoff}, we have for any $f \in \rmc(\msx, \ooint{0,+\infty})$ 
  \begin{equation}
    \label{eq:tres}
    d_H(\mce_{\pi_0}^x(f), u) = \log(\sup \ensembleLigne{\mce_{\pi_0}^x(f)(y)}{y \in \msy}) - \log(\inf \ensembleLigne{\mce_{\pi_0}^x(f)(y)}{y \in \msy}). 
  \end{equation}
  For any $f \in \rmc(\msx, \ooint{0,+\infty})$ and $y \in \msy$ we have
  \begin{equation}
    \textstyle{\mce_{\pi_0}^x(f)(y) \geq \exp[-\norm{c}_\infty] \int_{\msx} f(x) \rmd \pi_0(x), \quad \mce_{\pi_0}^x(f)(y) \leq \exp[\norm{c}_\infty] \int_{\msx} f(x) \rmd \pi_0(x).}
  \end{equation}
  Combining this result, \eqref{eq:uno}, \eqref{eq:dos} and \eqref{eq:tres} we
  get that $\Delta(\mce_{\pi_0}^x) \leq \tanh(\norm{c}_\infty)$. The proof that
  $\Delta(\mce_{\pi_1}^y) \leq \tanh(\norm{c}_\infty)$ is similar.  Lipschitz
  continuity follows from the definitions of $\mce_{\pi_0}^x, \mce_{\pi_1}^y$
  and the Lipschitz continuity of $K$.  In fact, for any function
  $f: \ \msx \to \rset$, resp.\ $g: \ \msy \to \rset$, that does not vanish
  $\pi_0$ a.e., resp.\ $\pi_1$-a.e., $y\mapsto \mce_{\pi_0}^x(f)(y)$, resp.
  $x\mapsto \mce_{\pi_0}^x(g)(x)$, is Lipschitz continuous.
\end{proof}

\begin{proposition}
  \label{prop:bound_wass_mce}
  Let $\pi_0, \hat{\pi}_0 \in \Pens(\msx)$ and
  $\pi_1, \hat{\pi}_1 \in \Pens(\msy)$. Then for any 
  $f \in \Lip(\msx, \ooint{0,+\infty})$ and
  $g \in \Lip(\msy, \ooint{0,+\infty})$ we have 
  \begin{align}
    &d_H(\mce^x_{\pi_0}(f), \mce^x_{\hat{\pi}_0}(f)) \leq 2  \norm{1/f}_\infty \left[\Lip(f)  + \Lip(c)  \norm{f}_\infty\right]\exp[2\norm{c}_\infty] \wassersteinD[1](\pi_0, \hat{\pi}_0), \\
    &d_H(\mce^y_{\pi_1}(g), \mce^y_{\hat{\pi}_1}(g)) \leq 2 \norm{1/g}_\infty  \left[\Lip(g)  + \Lip(c)  \norm{g}_\infty\right]\exp[2\norm{c}_\infty] \wassersteinD[1](\pi_1, \hat{\pi}_1). 
  \end{align}
\end{proposition}

\begin{proof}
  Let $f \in \Lip(\msx, \ooint{0,+\infty})$. We have
  \begin{equation}\label{eq:decomporatio}
    \textstyle{\mce_{\pi_0}^x(f)(y)/\mce_{\hat{\pi}_0}^x(f)(y) = 1 + \int_{\msx} K(x,y) f(x) \rmd (\pi_0 -  \hat{\pi}_0)(x) / \int_{\msx} K(x,y) f(x) \rmd \hat{\pi}_0(x). }
  \end{equation}
In addition, we have for any $x, x' \in \msx$ and  $y \in \msy$
\begin{align}
&    |K(x,y)f(x) - K(x',y) f(x')|\\
    &\qquad\leq |K(x,y) f(x)- K(x',y)f(x)| + |K(x',y) f(x) - K(x',y) f(x')|\\
    &\qquad\leq \|f\|_\infty \Lip(K(\cdot, y)) \ddx(x,x')  
    + \|K(\cdot, \cdot)\|_\infty \Lip(f) \ddx(x,x').
\end{align}
Since $K(x,y)=\exp[-c(x,y)]$, using the fact that for $|s|,|t|<M$ we have $|\exp[s]-\exp[t]|\leq \exp[M]|t-s|$, we have that for any $x, x' \in \msx$ and  $y \in \msy$
$$|K(x',y)-K(x,y)| \leq \exp[\|c\|_\infty] |c(x',y)-c(x,y)| \leq
 \exp[\|c\|_\infty] \Lip(c)\ddx(x,x').$$
 Therefore we have that for all $y\in\mcy$
 $$\Lip( K(\cdot, y) f(\cdot))\leq \|f\|_\infty \exp[\|c\|_\infty] \Lip(c) + \exp[\|c\|_\infty] \Lip(f).$$
Using this result we get that for any $y \in \mcy$
\begin{equation}
  \label{eq:numerator}
  \textstyle{\abs{\int_{\msx} K(x,y) f(x) \rmd (\pi_0 -  \hat{\pi}_0)(x)} \leq \left[\Lip(f)  + \Lip(c)  \norm{f}_\infty\right]\exp[\norm{c}_\infty] \wassersteinD[1](\pi_0, \hat{\pi}_0).}
\end{equation}
In addition, we have that for any $y \in \msy$
\begin{equation}
  \label{eq:denominator}
  \textstyle{\int_{\msx} K(x,y) f(x) \rmd x \geq \exp[-\norm{c}_\infty] / \norm{1/f}_\infty. }
\end{equation}
Combining \eqref {eq:decomporatio}, \eqref{eq:numerator} and \eqref{eq:denominator} we get that for any $y \in \msy$
\begin{align}
  \label{eq:upper_bound_left}
  \textstyle{\mce_{\pi_0}^x(f)(y)/\mce_{\hat{\pi}_0}^x(f)(y)  }
  &  \textstyle{\leq 1 + \norm{1/f}_\infty \left[\Lip(f)  + \Lip(c)  \norm{f}_\infty\right]\exp[2\norm{c}_\infty]\wassersteinD[1](\pi_0, \hat{\pi}_0). } 
\end{align}
Similarly, we have that for any $y \in \msy$
\begin{align}
  \label{eq:upper_bound_right}
  \textstyle{\mce_{\hat{\pi_0}}^x(f)(y)/\mce_{\pi_0}^x(f)(y)  }
  &  \textstyle{\leq 1 + \norm{1/f}_\infty \left[\Lip(f)  + \Lip(c)  \norm{f}_\infty\right]\exp[2\norm{c}_\infty] \wassersteinD[1](\pi_0, \hat{\pi}_0). } 
\end{align}
Combining \Cref{prop:hilbert_birkhoff}, \eqref{eq:upper_bound_left},
\eqref{eq:upper_bound_right} and the fact that for any $t \geq 0$,
$\log(1 + t) \leq t$ we get that
\begin{equation}
  d_H(\mce_{\pi_0}^x(f), \mce_{\hat{\pi}_0}^x(f)) \leq 2 \norm{1/f}_\infty \left[\Lip(f)  + \Lip(c)  \norm{f}_\infty\right]\exp[2\norm{c}_\infty]\wassersteinD[1](\pi_0, \hat{\pi}_0).
\end{equation}
The proof for $d_H(\mce^y_{\pi_1}(g), \mce^y_{\hat{\pi}_1}(g))$ is similar.
\end{proof}

\subsection{Quantitative uniform bounds on the potentials}
\label{sec:quant-unif-bounds}

In this section, we derive quantitative uniform bounds on the potentials w.r.t.\ 
the Hilbert--Birkhoff metric. More precisely, we show the following theorem.

\begin{theorem}
  \label{thm:contract_hilbert}
  For any $\pi_0, \hat{\pi}_0 \in \Pens(\msx)$,
  $\pi_1, \hat{\pi}_1 \in \Pens(\msy)$ let $(f_n, g_n)_{n\in \nset}$ and
  $(\hat{f}_n, \hat{g}_n)_{n\in \nset}$ the rescaled IPFP potential associated
  with $(\pi_0, \pi_1)$, respectively $(\hat{\pi}_0, \hat{\pi}_1)$ and given by
  \eqref{eq:potentials_rescale_form}.
  Then, for any $n \in \nset$ we have
  \begin{equation}
    d_H(f_ng_n, \hat{f}_n\hat{g}_n) \leq 8 \Lip(c) \rme^{10 \norm{c}_\infty} (\wassersteinD[1](\pi_0, \hat{\pi}_0) + \wassersteinD[1](\pi_1, \hat{\pi}_1)). 
  \end{equation}
\end{theorem}

\begin{proof}
  Let $n \in \nset$,
  $\rmD^x: \ \rmc(\msx, \ooint{0,+\infty}) \to \rmc(\msx, \ooint{0,+\infty})$
  and
  $\rmD^y: \ \rmc(\msy, \ooint{0,+\infty}) \to \rmc(\msy, \ooint{0,+\infty})$
  such that for any $f \in \rmc(\msx, \ooint{0,+\infty})$ and
  $g \in \rmc(\msy, \ooint{0,+\infty})$ we have $\rmD^x(f) = 1/f$ and
  $\rmD^y(g) = 1/g$.  We define
  $\mce_{\pi_1, \pi_0}^y : \ \rmc(\msy, \ooint{0,+\infty}) \to \rmc(\msx,
  \ooint{0,+\infty})$ such that for any $g \in \rmc(\msy, \ooint{0,+\infty})$ we
  have
  \begin{equation}
    \textstyle{\mce_{\pi_1, \pi_0}^y(g) = \mce_{\pi_1}^y(g) \exp[\int_{\msx}\log(1/\mce_{\pi_1}^y(g)(x)) \rmd \pi_0(x)]. }
  \end{equation}
  Using \eqref{eq:potentials_rescale_form}, we have for any $n \in \nset$
  \begin{align}
    \label{eq:iteation_brikhoff}
    &f_{n+1} = \rmD^x \circ \mce_{\pi_1, \pi_0}^y \circ \rmD^y \circ \mce_{\pi_0}^x(f_n), \\
    &\hat{f}_{n+1} = \rmD^x \circ \mce_{\hat{\pi}_1, \hat{\pi}_0}^y \circ \rmD^y \circ \mce_{\hat{\pi}_0}^x(\hat{f}_n),\\
    &g_{n+1} = \rmD^y \circ \mce_{\pi_0}^x \circ \rmD^x \circ \mce_{\pi_1, \pi_0}^y (g_n), \\
    &\hat{g}_{n+1} =  \rmD^y \circ \mce_{\hat{\pi}_0}^x\circ 
    \rmD^x \circ \mce_{\hat{\pi}_1, \hat{\pi}_0}^y 
    (\hat{g}_n).
  \end{align}
  Using the triangle inequality and \Cref{prop:hilbert_birkhoff} we have for any $n \in \nset$ 
  \begin{equation}
    \label{eq:hilbert_birkhoff_prod}
    d_H(f_n g_n, \hat{f}_n \hat{g}_n) \leq d_H(f_n g_n, f_n \hat{g}_n) + d_H(f_n \hat{g}_n, \hat{f}_n \hat{g}_n) \leq d_H(g_n, \hat{g}_n) + d_H(f_n, \hat{f}_n). 
  \end{equation}
  Recall that $f_0 = \hat{f}_0 = 1$ and therefore $d_H(f_0, \hat{f}_0) = 0$.
  Using \Cref{prop:hilbert_birkhoff}, \Cref{prop:birkhoff_contraction},
  \eqref{eq:iteation_brikhoff} and the fact that $\rmD^x, \rmD^y$ are isometries, we have for any $n \in \nset$
  \begin{align}
    &d_H(f_{n+1}, \hat{f}_{n+1}) = d_H(\rmD^x \circ \mce_{\pi_1, \pi_0}^y \circ \rmD^y \circ \mce_{\pi_0}^x(f_n), \rmD^x \circ \mce_{\hat{\pi}_1, \hat{\pi}_0}^y \circ \rmD^y \circ \mce_{\hat{\pi}_0}^x(\hat{f}_n)) \\
                                & \quad = d_H(\mce_{\pi_1, \pi_0}^y \circ \rmD^y \circ \mce_{\pi_0}^x(f_n), \mce_{\hat{\pi}_1, \hat{\pi}_0}^y \circ \rmD^y \circ \mce_{\hat{\pi}_0}^x(\hat{f}_n)) \\
                                & \quad = d_H(\mce_{\pi_1}^y \circ \rmD^y \circ \mce_{\pi_0}^x(f_n), \mce_{\hat{\pi}_1}^y \circ \rmD^y \circ \mce_{\hat{\pi}_0}^x(\hat{f}_n)) \\
                                & \quad \leq d_H(\mce_{\pi_1}^y \circ \rmD^y \circ \mce_{\pi_0}^x(f_n), \mce_{\hat{\pi}_1}^y \circ \rmD^y \circ \mce_{\pi_0}^x(f_n)) \\
                                & \quad  \qquad + d_H(\mce_{\hat{\pi}_1}^y \circ \rmD^y \circ \mce_{\pi_0}^x({f}_n), \mce_{\hat{\pi}_1}^y \circ \rmD^y \circ \mce_{\hat{\pi}_0}^x(\hat{f}_n)) \\
    & \quad \leq d_H(\mce_{\pi_1}^y \circ \rmD^y \circ \mce_{\pi_0}^x(f_n), \mce_{\hat{\pi}_1}^y \circ \rmD^y \circ \mce_{\pi_0}^x(f_n)) + \kappa d_H(\mce_{\pi_0}^x(f_n), \mce_{\hat{\pi}_0}^x(\hat{f}_n)) \\
    & \quad \leq d_H(\mce_{\pi_1}^y \circ \rmD^y \circ \mce_{\pi_0}^x(f_n), \mce_{\hat{\pi}_1}^y \circ \rmD^y \circ \mce_{\pi_0}^x(f_n)) \\
    & \qquad \quad + \kappa d_H(\mce_{\hat{\pi}_0}^x(\hat{f}_n), \mce_{\pi_0}^x(\hat{f}_n)) + \kappa^2 d_H(f_n, \hat{f}_n)  \\
    & \quad = d_H(\mce_{\pi_1}^y (g_n), \mce_{\hat{\pi}_1}^y (g_n))  +\kappa  d_H(\mce_{\hat{\pi}_0}^x(\hat{f}_n), \mce_{\pi_0}^x(\hat{f}_n)) + \kappa^2 d_H(f_n, \hat{f}_n),    \label{eq:ineq_fonda}
  \end{align}
with $\kappa = \tanh(\norm{c}_\infty) \geq \max\{ \kappa(\mce_{\hat{\pi}_1}^y),\kappa(\mce_{\hat{\pi}_0}^y) \}$. 
Using \Cref{prop:bound_wass_mce} we have for any $n \in \nset$
\begin{align}
  &d_H(\mce_{\hat{\pi}_0}^x(\hat{f}_n), \mce_{\pi_0}^x(\hat{f}_n)) \leq 2 \normLigne{1/\hat f_n}_\infty ( \Lip(\hat f_n)  + \Lip(c) \normLigne{\hat f_n}_\infty)\exp[2\normLigne{c}_\infty] \wassersteinD[1](\pi_0, \hat{\pi}_0), \\
  &d_H(\mce_{\pi_1}^y (g_n), \mce_{\hat{\pi}_1}^y (g_n)) \leq 2 \normLigne{1/g_n}_\infty (\Lip(g_n)  + \Lip(c)  \normLigne{g_n}_\infty)\exp[2\normLigne{c}_\infty] \wassersteinD[1](\pi_1, \hat{\pi}_1).
      \label{eq:n_step_bound}
\end{align}
Combining this result and \Cref{prop:control_potential}, we have
 for any $n \in \nset$
\begin{align}
  &d_H(\mce_{\hat{\pi}_0}^x(\hat{f}_n), \mce_{\pi_0}^x(\hat{f}_n)) \leq 4 \Lip(c) \rme^{8 \norm{c}_\infty} \wassersteinD[1](\pi_0, \hat{\pi}_0), \\
  &d_H(\mce_{\pi_1}^y (g_n), \mce_{\hat{\pi}_1}^y (g_n))  \leq 4 \Lip(c) \rme^{8 \norm{c}_\infty} \wassersteinD[1](\pi_1, \hat{\pi}_1).
\end{align}
Combining this result and \eqref{eq:ineq_fonda} we get that for any $n \in \nset$
\begin{equation}
  d_H(f_{n+1}, \hat{f}_{n+1}) \leq \tanh(\norm{c}_\infty) d_H(f_n, \hat{f}_n) + 4 \Lip(c) \rme^{8 \norm{c}_\infty} (\wassersteinD[1](\pi_0, \hat{\pi}_0) + \wassersteinD[1](\pi_1, \hat{\pi}_1)).
\end{equation}
Since $d_H(f_0, \hat{f}_0) = 0$ we have that for any $n \in \nset$
\begin{align}
  d_H(f_{n+1}, \hat{f}_{n+1}) & \textstyle{\leq 4 \Lip(c) \rme^{8 \norm{c}_\infty} \sum_{k=0}^n \tanh(\norm{c}_\infty)^k (\wassersteinD[1](\pi_0, \hat{\pi}_0) + \wassersteinD[1](\pi_1, \hat{\pi}_1)) }
  \\
  & \textstyle{\leq 4 \Lip(c) \rme^{8 \norm{c}_\infty} (1-\tanh(\|c\|_\infty))^{-1} (\wassersteinD[1](\pi_0, \hat{\pi}_0) + \wassersteinD[1](\pi_1, \hat{\pi}_1)) }
  \\
                              &\textstyle{ \leq 2  \Lip(c) \rme^{8 \norm{c}_\infty} (1 + \rme^{2\norm{c}_\infty}) (\wassersteinD[1](\pi_0, \hat{\pi}_0) + \wassersteinD[1](\pi_1, \hat{\pi}_1))} \\
  & \textstyle{ \leq 4 \Lip(c) \rme^{10 \norm{c}_\infty} (\wassersteinD[1](\pi_0, \hat{\pi}_0) + \wassersteinD[1](\pi_1, \hat{\pi}_1)).
    }   \label{eq:ineq_fn}  
\end{align}
Similarly, we get that for any $n \in \nset$
\begin{equation}
  \label{eq:ineq_gn}
  \textstyle{d_H(g_n, \hat{g}_n}) \leq 4 \Lip(c) \rme^{10 \norm{c}_\infty} (\wassersteinD[1](\pi_0, \hat{\pi}_0) + \wassersteinD[1](\pi_1, \hat{\pi}_1)).
\end{equation}
Combining \eqref{eq:hilbert_birkhoff_prod}, \eqref{eq:ineq_fn} and  \eqref{eq:ineq_gn} concludes the proof.
\end{proof}

Unfortunately controlling $d_H$ is not enough to control the distance between
$\Pbb^{n}$ and $\hat{\Pbb}^n$ for any $n \in \nset$. Indeed, using the
Hilbert--Birkhoff metric we control the oscillations of
$f_n g_n / (\hat{f}_n \hat{g}_n)$ for any $n \in \nset$ but in order to control
probability distances between $\Pbb^{n}$ and $\hat{\Pbb}^n$ for any
$n \in \nset$ we need to derive an upper-bound for
$\normLigne{f_n g_n - \hat{f}_n \hat{g}_n}_\infty$ for any $n \in \nset$. The
next lemma is key in order to obtain such bounds.

\begin{lemma}
  \label{lemma:attain}
  For any $\pi_0, \hat{\pi}_0 \in \Pens(\msx)$,
  $\pi_1, \hat{\pi}_1 \in \Pens(\msy)$ let $(f_n, g_n)_{n\in \nset}$ and
  $(\hat{f}_n, \hat{g}_n)_{n\in \nset}$ the rescaled IPFP potential associated
  with $(\pi_0, \pi_1)$, respectively $(\hat{\pi}_0, \hat{\pi}_1)$ and given by
  \eqref{eq:potentials_rescale_form}.  Then, for any $n \in \nset$ there exist
  $x_n^\dagger \in \msx$ and $y_n^\dagger \in \msy$ such that
  \begin{align}
    &|f_n(x_n^\dagger) g_n(y_n^\dagger)/ (\hat{f}_n(x_n^\dagger) \hat{g}_n(y_n^\dagger)) - 1| \leq  4 \Lip(c) \rme^{10 \norm{c}_\infty} (\wassersteinD[1](\pi_0, \hat{\pi}_0) + \wassersteinD[1](\pi_1, \hat{\pi}_1))  .
  \end{align}
\end{lemma}

\begin{proof}
  Let $n \in \nset$ and 
  $$\Delta := 4 \Lip(c) \rme^{10 \norm{c}_\infty} (\wassersteinD[1](\pi_0, \hat{\pi}_0) + \wassersteinD[1](\pi_1, \hat{\pi}_1)).$$
  Using that $\Pbb^{2n}(\msx \times \msy) =\hat{\Pbb}^{2n}(\msx \times \msy) =1$, we have
  \begin{align}
    \label{eq:diff_up}
    &\textstyle{ \int_{\msx \times \msy} \{f_n(x) g_n(y) / (\hat{f}_n(x) \hat{g}_n(y)) - 1\}\hat{f}_n(x) \hat{g}_n(y)\rmd \hat{\pi}_0(x)\rmd \hat{\pi}_1(y) } \\
    &\qquad \qquad \qquad = \textstyle{\int_{\msx \times \msy} f_n(x) g_n(y) \rmd \hat{\pi}_0(x)\rmd \hat{\pi}_1(y)} -1 \\
      &\qquad \qquad \qquad = \textstyle{\int_{\msx \times \msy} f_n(x) g_n(y) \rmd \hat{\pi}_0(x)\rmd \hat{\pi}_1(y) - \int_{\msx \times \msy} f_n(x) g_n(y) \rmd \pi_0(x)\rmd \pi_1(y)}  . 
  \end{align}
  In addition, using \Cref{prop:control_potential} we have
  \begin{equation}
       \Lip (f_n g_n) \leq  \|f_n\|_\infty \Lip(g_n) + \|g_n\|_\infty \Lip(f_n) \leq 2 \Lip(c) \rme^{6 \normLigne{c}_\infty}  . 
  \end{equation}
Combining this result, \eqref{eq:diff_up} and the fact that $\wassersteinD[1](\pi_0 \otimes \pi_1, \hat{\pi}_0 \otimes \hat{\pi}_1)\leq
\wassersteinD[1](\pi_0, \hat{\pi}_0) + \wassersteinD[1](\pi_1, \hat{\pi}_1)$, we get that
\begin{align}
&  \abs{\textstyle{ \int_{\msx \times \msy} \{f_n(x) g_n(y) / (\hat{f}_n(x) \hat{g}_n(y)) - 1\}\rmd \hat{\Pbb}^{2n}(x,y) }}\\
  & \qquad \leq 2 \Lip(c) \rme^{6 \normLigne{c}_\infty} \wassersteinD[1](\pi_0 \otimes \pi_1, \hat{\pi}_0 \otimes \hat{\pi}_1) \\
   & \qquad \leq 2 \Lip(c) \rme^{6 \normLigne{c}_\infty} (\wassersteinD[1](\pi_0, \hat{\pi}_0) + \wassersteinD[1](\pi_1, \hat{\pi}_1) ) < \Delta/2.
\end{align}
Therefore
\begin{align}
  \label{eq:ineq_true}
  & 1 - \Delta /2 \leq \textstyle{ \int_{\msx \times \msy} f_n(x) g_n(y) / (\hat{f}_n(x) \hat{g}_n(y))\rmd \hat{\Pbb}^{2n}(x,y) \leq 1 + \Delta/2.}
\end{align}
Assume that for any $(x,y) \in \msx \times \msy$ we have that
\begin{equation}
  \absLigne{f_n(x)g_n(y) / (\hat{f}_n(x) \hat{g}_n(y)) - 1} > \Delta.
\end{equation}
Combining this with \eqref{eq:ineq_true} there exist
$(x_n^+, y_n^+) \in \msx \times \msy$ and $(x_n^-, y_n^-) \in \msx \times \msy$
such that
\begin{align}
    f_n(x_n^+) g_n(y_n^+) / (\hat{f}_n(x_n^+) \hat{g}_n(y_n^+)) > 1+  \Delta, \quad
    f_n(x_n^-) g_n(y_n^-) / (\hat{f}_n(x_n^-) \hat{g}_n(y_n^-)) < 1 -  \Delta,
\end{align}
whence it follows that $\Delta <1$.
Therefore, by \Cref{thm:contract_hilbert} we get that
\begin{align}
  &d_H(f_n g_n, \hat{f}_n \hat{g}_n)\\
   &\qquad\geq \log(f_n(x_n^+) g_n(y_n^+) / (\hat{f}_n(x_n^+) \hat{g}_n(y_n^+))) - \log(f_n(x_n^-) g_n(y_n^-) / (\hat{f}_n(x_n^-) \hat{g}_n(y_n^-))) \\
  &\qquad> \log((1 + \Delta)/(1 - \Delta)) \geq 2\Delta \geq d_H(f_n g_n, \hat{f}_n \hat{g}_n) ,
\end{align}
which is absurd.
\end{proof}

Finally, we conclude this section by deriving bounds on
$\normLigne{f_n g_n  - \hat{f}_n\hat{g}_n}_\infty$ combining
\Cref{thm:contract_hilbert} with \Cref{lemma:attain}.

\begin{theorem}
  \label{thm:contrat_infty}
  For any $\pi_0, \hat{\pi}_0 \in \Pens(\msx)$,
  $\pi_1, \hat{\pi}_1 \in \Pens(\msy)$ let $(f_n, g_n)_{n\in \nset}$ and
  $(\hat{f}_n, \hat{g}_n)_{n\in \nset}$ the rescaled IPFP potential associated
  with $(\pi_0, \pi_1)$, respectively $(\hat{\pi}_0, \hat{\pi}_1)$ and given by
  \eqref{eq:potentials_rescale_form}. Then, for any $n \in \nset$ we have
  \begin{equation}
    \normLigne{f_ng_n -\hat{f}_n \hat{g}_n}_\infty \leq 12 \Lip(c) \rme^{16 \normLigne{c}_\infty} (\wassersteinD[1](\pi_0, \hat{\pi}_0) + \wassersteinD[1](\pi_1, \hat{\pi}_1) )  . 
  \end{equation}
\end{theorem}

\begin{proof}
  Let $n \in \nset$, $x \in \msx$ and $y \in \msy$. Using \Cref{prop:bound_0}
  and the fact that for any $s,t \in \cball{0}{M}$ with $M \geq 0$ we have
  $\abs{\rme^s - \rme^t} \leq \rme^M \abs{s-t}$ we get 
  \begin{equation}
    \label{eq:bound_u}
  \absLigne{f_n(x)g_n(y) - \hat{f}_n(x) \hat{g}_n(y)} \leq \rme^{6 \normLigne{c}_\infty} \absLigne{\log(f_n(x)g_n(y)/(\hat{f}_n(x)\hat{g}_n(y)))}
\end{equation}
Assume that $f_n(x)g_n(y)/(\hat{f}_n(x)\hat{g}_n(y)) \geq 1$. Then using that
for any $t > 0$, $\log(t) \leq t - 1$, \Cref{thm:contract_hilbert} and
\Cref{lemma:attain} we have, with $(x_n^\dagger, y_n^\dagger)$ from \Cref{lemma:attain},
\begin{align}
  &\absLigne{\log(f_n(x)g_n(y)/(\hat{f}_n(x)\hat{g}_n(y)))} = \log(f_n(x)g_n(y)/(\hat{f}_n(x)\hat{g}_n(y))) \\
  & \qquad \quad + \log(\hat{f}_n(x_n^\dagger)\hat{g}_n(y_n^\dagger)/(f_n(x_n^\dagger)g_n(y_n^\dagger))) + \log(f_n(x_n^\dagger)g_n(y_n^\dagger)/(\hat{f}_n(x_n^\dagger)\hat{g}_n(y_n^\dagger))) \\
  &\qquad \leq d_H(f_n g_n, \hat{f}_n \hat{g}_n) + \log(f_n(x_n^\dagger)g_n(y_n^\dagger)/(\hat{f}_n(x_n^\dagger)\hat{g}_n(y_n^\dagger))) \\
  &\qquad \leq d_H(f_ng_n, \hat{f}_n \hat{g}_n) + f_n(x_n^\dagger)g_n(y_n^\dagger)/(\hat{f}_n(x_n^\dagger)\hat{g}_n(y_n^\dagger)) - 1 \\
   & \qquad \leq 12 \Lip(c)\rme^{10 \normLigne{c}_\infty} (\wassersteinD[1](\pi_0, \hat{\pi}_0) + \wassersteinD[1](\pi_1, \hat{\pi}_1) )  . 
\end{align}

Combining this result and \eqref{eq:bound_u} we get that
\begin{equation}
  \absLigne{f_n(x)g_n(y) - \hat{f}_n(x) \hat{g}_n(y)} \leq 12 \Lip(c) \rme^{16 \normLigne{c}_\infty} (\wassersteinD[1](\pi_0, \hat{\pi}_0) + \wassersteinD[1](\pi_1, \hat{\pi}_1) )  .
\end{equation}
The proof in the case where $f_n(x)g_n(y)/(\hat{f}_n(x)\hat{g}_n(y)) \leq 1$ is similar.
\end{proof}
\subsection{From potentials to probability metrics}
\label{sec:from-potent-prob}

Using \Cref{thm:contrat_infty} we are now ready to prove \Cref{thm:stability_ipfp}.


\begin{proof}[Proof of \Cref{thm:stability_ipfp}]
  Let $n \in \nset$ and $F\in \Lip(\msx\times\msy, \rset)$, that is $F: \ \msx \times \msy \to \rset$ such that for any
  $x_0,x_1 \in \msx$ and $y_0, y_1 \in \msy$ we have
  \begin{equation}
    |F(x_0,y_0) - F(x_1,y_1)| \leq \ddx(x_0,x_1)+\ddy(y_0,y_1). 
  \end{equation}
  We will be considering quantities of the form $\int_{\msx \times \msy} F(x,y) [\rmd \mu - \rmd \mu'](x,y)$ where $\mu, \mu' \in \Pens(\msx\times\msy)$;  
  therefore possibly replacing $F$ with $F-a$ for some constant $a$, we may assume that there exist $\bar{x} \in \msx$ and
  $\bar{y} \in \msy$ such that $F(\bar{x}, \bar{y}) = 0$. Therefore, we have
   that 
  \begin{equation}
    \label{eq:inf_Psi}
    \normLigne{F}_\infty = \sup \ensembleLigne{\abs{F(x,y) - F(\bar{x},\bar{y})}}{x \in \msx, \ y \in \msy} \leq \diam_\msx +  \diam_\msy.
  \end{equation}
  We write
  $\Lipset^\star = \ensembleLigne{F \in \Lipset_1}{F(\bar{x}, \bar{y}) =0}$.
  Using this result and \Cref{prop:control_potential}, we get that
  \begin{align}
      \Lip(F K f_n g_n)
      &\leq \Lip(F)\|K\|_\infty \|f_n\|_\infty \|g_n\|_\infty 
      + \Lip(K) \|F\|_\infty \|f_n\|_\infty\|g_n\|_\infty\\
    &\qquad       +\Lip(f_n) \|F\|_\infty \|K\|_\infty\|g_n\|_\infty + \Lip(g_n)
      \|F \|_\infty\|K\|_\infty \|f_n\|_\infty\\
      &\leq \rme^{7\|c\|_\infty} + 3\Lip(c)\rme^{7\|c\|_\infty}(\diam_\msx+\diam_\msy).
  \end{align}
Combining this result with \Cref{thm:contrat_infty} and the fact that $\wassersteinD[1](\pi_0 \otimes \pi_1, \hat{\pi}_0 \otimes \hat{\pi}_1)\leq
\wassersteinD[1](\pi_0, \hat{\pi}_0) + \wassersteinD[1](\pi_1, \hat{\pi}_1)$, we get that
\begin{align}
\label{eq:bound_wass_uno}
  &\textstyle{\int_{\msx \times \msy} F(x,y)K(x,y) f_n(x) g_n(y) \rmd \pi_0(x) \rmd \pi_1(y) - \int_{\msx \times \msy}F(x,y)K(x,y) \hat f_n(x) \hat g_n(y) \rmd \hat{\pi}_0(x) \rmd \hat{\pi}_1(y) }  \\
  &\leq \textstyle{\int_{\msx \times \msy} F(x,y)K(x,y) f_n(x) g_n(y) \rmd \pi_0(x) \rmd \pi_1(y) - \int_{\msx \times \msy}F(x,y)K(x,y) f_n(x) g_n(y) \rmd \hat{\pi}_0(x) \rmd \hat{\pi}_1(y) }  \\
  &\qquad\qquad  + \textstyle{\int_{\msx \times \msy} F(x,y)K(x,y) \sup\|f_n g_n- \hat f_n \hat g_n\|_\infty \rmd \hat\pi_0(x) \rmd \hat\pi_1(y) }\\
  &\leq \left[ \Lip(FK f_n g_n)+ 12(\diam_\msx +\diam_\msy)\Lip(c) \rme^{17 \normLigne{c}_\infty} \right] [\wassersteinD[1](\pi_0, \hat{\pi}_0) + \wassersteinD[1](\pi_1, \hat{\pi}_1) ]\\
  &\leq \left(  \rme^{7\|c\|_\infty} + 3\Lip(c)(\diam_\msx+\diam_\msy)\rme^{7\|c\|_\infty}+ 12(\diam_\msx +\diam_\msy)\Lip(c) \rme^{17 \normLigne{c}_\infty} \right) [\wassersteinD[1](\pi_0, \hat{\pi}_0) + \wassersteinD[1](\pi_1, \hat{\pi}_1) ]\\
  & \leq \rme^{17 \normLigne{c}_\infty} \{1 + 15\Lip(c) (\diam_\msx + \diam_\msy) \} (\wassersteinD[1](\pi_0, \hat{\pi}_0) + \wassersteinD[1](\pi_1, \hat{\pi}_1)). 
\end{align}

Therefore, we have that 
\begin{align}
  \wassersteinD[1](\Pbb^{2n}, \hat{\Pbb}^{2n}) &= \sup \ensembleLigne{\textstyle{\int_{\msx \times \msy} F(x,y) \rmd \Pbb^{2n}(x,y) - \int_{\msx \times \msy} F(x,y) \rmd \hat{\Pbb}^{2n}(x,y)}}{F \in \Lipset} \\
                                               &= \sup \ensembleLigne{\textstyle{\int_{\msx \times \msy} F(x,y) \rmd \Pbb^{2n}(x,y) - \int_{\msx \times \msy} F(x,y) \rmd \hat{\Pbb}^{2n}(x,y)}}{F \in \Lipset^\star} \\
                                               &\leq \rme^{17 \normLigne{c}_\infty} \{1 + 15\Lip(c) (\diam_\msx + \diam_\msy) \}  (\wassersteinD[1](\pi_0, \hat{\pi}_0) + \wassersteinD[1](\pi_1, \hat{\pi}_1)).\qedhere
\end{align}
\end{proof}
\begin{proof}[Proof of \Cref{thm:stability_schro}]

  Let $n \in \nset$. Using \cite[Lemma 4]{chen2016entropic} we have that
  $d_H(f_{n+1},f_n)\leq \kappa^n d_H(f_1,f_0)$ and
  $d_H(g_{n+1},g_n)\leq \kappa^n d_H(g_1,g_0)$.  Thus
\begin{align}d_H(f_{n+1}g_{n+1}, f_n g_n)
  &= d_H(f_{n+1}, f_n) + d_H(g_{n+1},g_n) \leq \kappa^n \left[ d_H(f_1, f_0)+
    d_H(g_1, g_0)\right]. \label{eq:eq_1}
\end{align}
As explained earlier this is not enough on its own to control
$\|f_{n+1}g_{n+1}-f_n g_n\|_\infty$.  However, we can use a similar technique as
\Cref{lemma:attain}. We have that
\begin{equation}\label{eq:integralis1}
  \int_{\msx \times \msy} [f_{n+1}(x)g_{n+1}(y)/(f_n(x)g_n(y))]  K(x,y) f_n(x)g_n(y) \rmd \pi_0(x) \rmd \pi_1(y) = 1.
\end{equation}

In what follows, we assume that for all $(x,y) \in \msx \times \msy$,
$\abs{f_{n+1}(x)g_{n+1}(y)/(f_n(x)g_n(y)) - 1} \geq \Delta$ with
$\Delta = \kappa^n \left[ d_H(f_1, f_0)+ d_H(g_1, g_0)\right]/2$. Combining this
with \eqref{eq:integralis1}, we get that there exist
$(x^+_n, y^+_n), (x^-_n,y^-_n) \in \msx \times \msy$ such that
\begin{equation}
  f_{n+1}(x^+_n)g_{n+1}(y^+_n)/(f_n(x^+_n)g_n(y^+_n)) \geq 1 + \Delta, \quad f_{n+1}(x^-_n)g_{n+1}(y^-_n)/(f_n(x^-_n)g_n(y^-_n)) \leq 1 - \Delta.
\end{equation}
This implies that $\Delta <1$.  Hence, we have that 
\begin{align}
  &d_H(f_{n+1}g_{n+1}, f_ng_n) \\
  & \qquad \geq \log(f_{n+1}(x^+_n)g_{n+1}(y^+_n)/(f_n(x^+_n)g_n(y^+_n))) - \log(f_{n+1}(x^-_n)g_{n+1}(y^-_n)/(f_n(x^-_n)g_n(y^-_n))) \\
  & \qquad \geq \log((1+\Delta)/(1-\Delta)) > 2 \Delta \geq d_H(f_{n+1}g_{n+1}, f_ng_n).
\end{align}
This is absurd, hence there exists $(x^\star, y^\star) \in \msx \times \msy$ such that
\begin{equation}
  \abs{f_{n+1}(x^\star) g_{n+1}(y^\star) / (f_{n}(x^\star) g_{n}(y^\star)) - 1} \leq \kappa^n \left[ d_H(f_1, f_0)+
    d_H(g_1, g_0)\right]. 
\end{equation}
Therefore, we have that for any $x \in \msx$ and $y \in \msy$
\begin{align}
&\abs{\log\left(f_{n+1}(x)g_{n+1}(y)/ (f_n(x)  g_n(x))\right)}\\
&\qquad\leq \abs{\log\left(f_{n+1}(x^\star)g_{n+1}(y^\star)/(f_{n}(x^\star)g_{n}(y^\star))
\right)} + d_H(f_{n+1}g_{n+1}, f_n g_n)\\
&\qquad\leq 2 \kappa^n \left[ d_H(f_1, f_0)+ d_H(g_1, g_0)\right].
\end{align}
Combining this result, \Cref{prop:control_potential} and the fact that for any
$s,t \in \ccint{0,M}$, $\absLigne{\rme^t - \rme^s} \leq \rme^M \absLigne{t - s}$
we get that for any $x \in \msx$ and $y \in \msy$
\begin{align}
  \absLigne{f_{n+1}(x)g_{n+1}(y) - f_n(x)g_n(y)} &\leq \rme^{6 \normLigne{c}_\infty} \abs{\log\left(f_{n+1}(x)g_{n+1}(y)/ f_n(x) g_n(x)\right)} \\
  &\leq 2 \rme^{6 \normLigne{c}_\infty} \kappa^n \left[ d_H(f_1, f_0)+ d_H(g_1, g_0)\right].
\end{align}
Therefore, we have 
$$\|f_{n+1}g_{n+1}-f_n g_n\|_\infty \leq 2 \rme^{6 \normLigne{c}_\infty} \kappa^n \left[ d_H(f_1, f_0)+ d_H(g_1, g_0)\right].$$
Let $F \in \Lip_1(\msx\times \msy, \rset)$, 
and without loss of generality we may assume that $F(\bar x, \bar y)=0$ for a fixed pair $(\bar x, \bar y)\in \msx\times \msy$. 
We then have 
\begin{align}
&\int_{\msx \times \msy} F(x,y) f_{n+1}(x) g_{n+1}(y) K(x,y) \rmd \pi_0( x) \rmd\pi_1( y)\\
&\quad -   \int_{\msx \times \msy} F(x,y) f_{n}(x) g_{n}(y) K(x,y) \rmd \pi_0( x) \rmd\pi_1( y)\\
&\quad \leq \int_{\msx \times \msy} \|F(x,y)\|_\infty \|f_{n+1} g_{n+1} - f_n g_n\|_\infty \|K\|_\infty \rmd\pi_0( x) \rmd\pi_1( y)\\
&\quad \leq 2 (\diam_\msx + \diam_\msy)\rme^{7\|c\|_\infty}\kappa^n \left[ d_H(f_1, f_0)+ d_H(g_1, g_0)\right].
\end{align}
Taking the supremum over $\{F\in \Lip_1(\msx\times\msy, \rset): F(\bar x, \bar y) =0\}$, we have that 
$$\wassersteinD[1](\Pbb^{n+1}, \Pbb^{n}) \leq 2 (\diam_\msx + \diam_\msy)\rme^{7\|c\|_\infty}\kappa^n \left[ d_H(f_1, f_0)+ d_H(g_1, g_0)\right].$$
By completeness of $(\Pens(\msx\times\msy), \wassersteinD[1])$ we have that
$\Pbb^n$ converges in $(\Pens(\msx\times\msy), \wassersteinD[1])$ to
$\Pbb^\ast\in \Pens(\msx\times\msy)$.  Similarly
$\hat{\Pbb}^n \to \hat{\Pbb}^\ast\in\Pens_1(\msx\times\msy)$. Combining
these results and \Cref{thm:stability_ipfp} we have
\begin{align}
\wassersteinD[1](\Pbb^\ast, \hat \Pbb^\ast)
   &\leq \wassersteinD[1](\Pbb^\ast, \Pbb^n)
   + \wassersteinD[1](\Pbb^n, \hat\Pbb^n)+ 
   \wassersteinD[1](\hat\Pbb^n, \hat\Pbb^\ast)\\
   &\leq C \defEns{\wassersteinD[1](\pi_0, \hat{\pi}_0) + \wassersteinD[1](\pi_1, \hat{\pi}_1)} + \wassersteinD[1](\Pbb^\ast, \Pbb^n)+\wassersteinD[1](\hat\Pbb^n, \hat\Pbb^\ast).
\end{align}
We conclude upon letting $n \to +\infty$.
\end{proof}

\bibliographystyle{apalike}

\end{document}